\documentclass{article}

\usepackage{graphicx} 
\usepackage{subfigure}

\usepackage{natbib}

\usepackage{algorithm}
\usepackage{algorithmic}

\usepackage{hyperref}

\usepackage[accepted]{icml2013}

\icmltitlerunning{GIST: General Iterative Shrinkage and Thresholding algorithm}

\usepackage{graphicx}
\usepackage{amsmath,amssymb,url}
\usepackage{float}
\usepackage{bm}
\usepackage{multirow}

\def\sign{\mathrm{sign}}

\newcommand{\argmin}{\mathop{\arg\min}}

\newcommand{\EqRef}[1]{Eq.~(\ref{#1})}

\newcommand{\ThemRef}[1]{Theorem~\ref{#1}}
\newcommand{\LemmaRef}[1]{Lemma~\ref{#1}}

\newcommand{\FigRef}[1]{Figure~\ref{#1}}
\newcommand{\AlgRef}[1]{Algorithm~\ref{#1}}
\newcommand{\TabRef}[1]{Table~\ref{#1}}

\newtheorem{theorem}{Theorem}
\newtheorem{lemma}{Lemma}

\newtheorem{remark}{Remark}

\newenvironment{proof}[1][Proof]{\begin{trivlist}
\item[\hskip \labelsep {\bfseries #1}]}{\end{trivlist}}

\def\qed {{
\parfillskip=0pt 
\widowpenalty=10000 
\displaywidowpenalty=10000 
\finalhyphendemerits=0 
\leavevmode 
\unskip 
\nobreak 
\hfil 
\penalty50 
\hskip.2em 
\null 
\hfill 
$\square$
\par}} 

\begin{document}

\twocolumn[
\icmltitle{A General Iterative Shrinkage and Thresholding Algorithm for Non-convex Regularized Optimization Problems}


\icmlauthor{Pinghua Gong}{gph08@mails.tsinghua.edu.cn}
\icmlauthor{Changshui Zhang}{zcs@mail.tsinghua.edu.cn}
\icmladdress{State Key Laboratory on Intelligent Technology and Systems\\
       Tsinghua National Laboratory for Information Science and Technology (TNList)\\
       Department of Automation, Tsinghua University, Beijing 100084, China}
\icmlauthor{Zhaosong Lu}{zhaosong@sfu.ca}
\icmladdress{Department of Mathematics, Simon Fraser University, Burnaby, BC, V5A 1S6, Canada}
\icmlauthor{Jianhua Z. Huang}{jianhua@stat.tamu.edu}
\icmladdress{Department of Statistics, Texas A$\&$M University, TX 77843, USA}
\icmlauthor{Jieping Ye}{jieping.ye@asu.edu}
\icmladdress{Computer Science and Engineering, Arizona State University, Tempe, AZ 85287, USA}


\icmlkeywords{Sparse learning, Non-convex optimization, Iterative shrinkage and thresholding, DC programming}

\vskip 0.3in
]

\begin{abstract}
Non-convex sparsity-inducing penalties have recently received considerable attentions in sparse learning.
Recent theoretical investigations have demonstrated their superiority over
the convex counterparts in several sparse learning settings. However, solving the non-convex optimization problems
associated with non-convex penalties remains a big challenge. A commonly used approach is the Multi-Stage (MS) convex
relaxation (or DC programming), which relaxes the original non-convex problem to a sequence of convex problems.
This approach is usually not very practical for large-scale problems because its computational cost is a multiple
of solving a single convex problem. In this paper, we propose a General Iterative Shrinkage and Thresholding
(GIST) algorithm to solve the nonconvex
optimization problem for a large class of non-convex penalties. The GIST algorithm iteratively solves a proximal operator problem,
which in turn has a closed-form solution for many commonly used penalties. At each outer iteration of the algorithm,
we use a line search initialized by the Barzilai-Borwein (BB) rule that allows finding an appropriate step size quickly.
The paper also presents a detailed convergence analysis of the GIST algorithm. The efficiency of the proposed algorithm
is demonstrated by extensive experiments on large-scale data sets.
\end{abstract}

\section{Introduction}\label{sec:introduction}
Learning sparse representations has important applications in many areas of science and engineering. The use of
an $\ell_0$-norm regularizer leads to a sparse solution, however the $\ell_0$-norm regularized optimization problem is challenging to solve,
due to the discontinuity and non-convexity of the $\ell_0$-norm regularizer. The $\ell_1$-norm regularizer,
a continuous and convex surrogate, has been studied extensively in the literature \cite{tibshirani1996regression,efron2004least}
and has been applied successfully to many applications including signal/image processing,
biomedical informatics and computer vision \cite{shevade2003simple,wright2008robust,beck2009fast,wright2009sparse,ye2012sparse}.
Although the $\ell_1$-norm based sparse learning formulations have achieved great success, they have been shown to be
suboptimal in many cases \cite{candes2008enhancing,zhang2010analysis,zhang2012multi}, since the $\ell_1$-norm is
a loose approximation of the $\ell_0$-norm and
often leads to an over-penalized problem. To address this issue, many non-convex regularizers, interpolated between the $\ell_0$-norm
and the $\ell_1$-norm, have been proposed to better approximate the $\ell_0$-norm. They include $\ell_q$-norm ($0<q<1$) \cite{foucart2009sparsest},
Smoothly Clipped Absolute Deviation (SCAD) \cite{fan2001variable},
Log-Sum Penalty (LSP) \cite{candes2008enhancing}, Minimax Concave Penalty (MCP) \cite{zhang2010nearly},
Geman Penalty (GP) \cite{geman1995nonlinear,trzasko2009relaxed} and Capped-$\ell_1$ penalty \cite{zhang2010analysis,zhang2012multi,gong2012multi}.

Although the non-convex regularizers (penalties) are appealing in sparse learning, it is challenging to solve the corresponding non-convex
optimization problems. In this paper, we propose a General Iterative Shrinkage and Thresholding (GIST) algorithm
for a large class of non-convex penalties. The key step of the proposed algorithm is to compute a proximal operator, which has a closed-form solution
for many commonly used non-convex penalties. In our algorithm, we adopt the Barzilai-Borwein (BB) rule \cite{barzilai1988two}
to initialize the line search step size at each iteration, which greatly accelerates the convergence speed.
We also use a non-monotone line search criterion to further speed up the convergence of the algorithm.
In addition, we present a detailed convergence analysis for the proposed algorithm. Extensive experiments on large-scale real-world data sets
demonstrate the efficiency of the proposed algorithm.


\section{The Proposed Algorithm: GIST}\label{sec:gist}
\subsection{General Problems}
We consider solving the following general problem:
\begin{align}\label{eq:nonconvexopt}
\min_{\mathbf{w}\in\mathbb{R}^d}\left\{f(\mathbf{w})=l(\mathbf{w})+r(\mathbf{w})\right\}.
\end{align}
We make the following assumptions on the above formulation throughout the paper:
\begin{itemize}
\item[\textbf{A1}] $l(\mathbf{w})$ is continuously differentiable with Lipschitz continuous gradient, that is, there exists
a positive constant $\beta(l)$ such that
\begin{align}
\|\nabla l(\mathbf{w})-\nabla l(\mathbf{u})\|\leq \beta(l)\|\mathbf{w}-\mathbf{u}\|,\forall \mathbf{w},\mathbf{u}\in\mathbb{R}^d.\nonumber
\end{align}
\item[\textbf{A2}]$r(\mathbf{w})$ is a continuous function which is possibly \emph{non-smooth} and \emph{non-convex}, and can be rewritten as
the difference of two convex functions, that is,
\begin{align}
r(\mathbf{w})=r_1(\mathbf{w})-r_2(\mathbf{w}),\nonumber
\end{align}
where $r_1(\mathbf{w})$ and $r_2(\mathbf{w})$ are convex functions.
\item[\textbf{A3}] $f(\mathbf{w})$ is bounded from below.
\end{itemize}

\begin{remark}
We say that $\mathbf{w}^\star$ is a critical point of problem (\ref{eq:nonconvexopt}), if the following
holds \cite{toland1979duality,wright2009sparse}:
\begin{align}
\mathbf{0}\in\nabla l(\mathbf{w}^\star)+\partial r_1(\mathbf{w}^\star)-\partial r_2(\mathbf{w}^\star),\nonumber
\end{align}
where $\partial r_1(\mathbf{w}^\star)$ is the sub-differential of the function $r_1(\mathbf{w})$ at $\mathbf{w}=\mathbf{w}^\star$,
that is,
\begin{align}
\partial r_1(\mathbf{w}^\star)=\left\{\mathbf{s}:r_1(\mathbf{w})\geq r_1(\mathbf{w}^\star)+\langle\mathbf{s},\mathbf{w}-\mathbf{w}^\star\rangle,\forall\mathbf{w}\in\mathbb{R}^d\right\}.\nonumber
\end{align}
We should mention that the sub-differential is non-empty on any \emph{convex} function; this is why we make the assumption
that $r(\mathbf{w})$ can be rewritten as the difference of two \emph{convex} functions.
\end{remark}

\subsection{Some Examples}
Many formulations in machine learning satisfy the assumptions above. The following least square and
logistic loss functions are two commonly used ones which satisfy assumption \textbf{A1}:
\begin{align}
l(\mathbf{w})=\frac{1}{2n}\|X\mathbf{w}-\mathbf{y}\|^2~\mathrm{or}~\frac{1}{n}\sum_{i=1}^n\log\left(1+\exp(-y_i\mathbf{x}_i^T\mathbf{w})\right),\nonumber
\end{align}
where $X=[\mathbf{x}_1^T;\cdots;\mathbf{x}_n^T]\in\mathbb{R}^{n\times d}$ is a data matrix and $\mathbf{y}=[y_1,\cdots,y_n]^T\in\mathbb{R}^n$
is a target vector.
The regularizers (penalties) which satisfy the assumption \textbf{A2} are presented in \TabRef{tab:regularizerexamples}.
They are non-convex (except the $\ell_1$-norm) and extensively used in sparse learning.
The functions $l(\mathbf{w})$ and $r(\mathbf{w})$ mentioned above are nonnegative. Hence, $f$ is bounded from below and satisfies assumption \textbf{A3}.

\begin{table*}[tb]\vspace{-0.0cm}
\caption{Examples of regularizers (penalties) $r(\mathbf{w})$ satisfying the assumption \textbf{A2} and the corresponding convex functions $r_1(\mathbf{w})$ and $r_2(\mathbf{w})$.
$\lambda>0$ is the regularization parameter; $r(\mathbf{w})=\sum_ir_i(w_i),r_1(\mathbf{w})=\sum_ir_{1,i}(w_i),
r_2(\mathbf{w})=\sum_ir_{2,i}(w_i),[x]_+=\max(0,x)$.} \vspace{-0.1cm} \label{tab:regularizerexamples}
\begin{center}
\footnotesize{
\begin{tabular}{l|l|l|l}
\hline\hline
Name &$r_i(w_i)$ &$r_{1,i}(w_i)$  &$r_{2,i}(w_i)$ \\
\hline
$\ell_1$-norm &$\lambda|w_i|$ &$\lambda|w_i|$ &$0$ \\
\hline
LSP &$\lambda\log(1+|w_i|/\theta)~(\theta>0)$ &$\lambda|w_i|$ &$\lambda(|w_i|-\log (1+|w_i|/\theta))$ \\
\hline
SCAD &$\lambda\int_{0}^{|w_i|}\min\left(1,\frac{[\theta\lambda-x]_+}{(\theta-1)\lambda}\right)dx~(\theta>2)$ &$\lambda|w_i|$ &$\lambda\int_{0}^{|w_i|}\frac{[\min(\theta\lambda,x)-\lambda]_+}{(\theta-1)\lambda}dx$ \\
&$=\left\{
   \begin{array}{ll}
     \lambda|w_i|, & \mathrm{if}~|w_i|\leq\lambda, \\
     \frac{-w_i^2+2\theta\lambda|w_i|-\lambda^2}{2(\theta-1)}, & \mathrm{if}~\lambda<|w_i|\leq\theta\lambda, \\
     (\theta+1)\lambda^2/2, & \mathrm{if}~|w_i|>\theta\lambda.
   \end{array}
 \right.$
 & &$=\left\{
   \begin{array}{ll}
     0, & \mathrm{if}~|w_i|\leq\lambda, \\
     \frac{w_i^2-2\lambda|w_i|+\lambda^2}{2(\theta-1)}, & \mathrm{if}~\lambda<|w_i|\leq\theta\lambda, \\
     \lambda|w_i|-\frac{(\theta+1)\lambda^2}{2}, & \mathrm{if}~|w_i|>\theta\lambda.
   \end{array}
 \right.$\\
\hline
MCP &$\lambda\int_{0}^{|w_i|}\left[1-\frac{x}{\theta\lambda}\right]_+dx~(\theta>0)$ &$\lambda|w_i|$ &$\lambda\int_{0}^{|w_i|}\min(1,x/(\theta\lambda))dx$ \\
&$=\left\{
   \begin{array}{ll}
     \lambda|w_i|-w_i^2/(2\theta), & \mathrm{if}~|w_i|\leq\theta\lambda, \\
     \theta\lambda^2/2, & \mathrm{if}~|w_i|>\theta\lambda.
   \end{array}
 \right.$
& &$=\left\{
   \begin{array}{ll}
     w_i^2/(2\theta), & \mathrm{if}~|w_i|\leq\theta\lambda, \\
     \lambda|w_i|-\theta\lambda^2/2, & \mathrm{if}~|w_i|>\theta\lambda.
   \end{array}
 \right.$\\
\hline
Capped $\ell_1$ &$\lambda\min(|w_i|,\theta)~(\theta>0)$ &$\lambda|w_i|$  &$\lambda[|w_i|-\theta]_+$ \\
\hline\hline
\end{tabular}}
\end{center}\vspace{-0.2cm}
\end{table*}

\subsection{Algorithm}
Our proposed General Iterative Shrinkage and Thresholding (GIST) algorithm solves problem (\ref{eq:nonconvexopt})
by generating a sequence $\{\mathbf{w}^{(k)}\}$ via:
\begin{align}
\mathbf{w}^{(k+1)}=&\argmin_{\mathbf{w}}~l(\mathbf{w}^{(k)})+\langle\nabla l(\mathbf{w}^{(k)}),\mathbf{w}-\mathbf{w}^{(k)}\rangle\nonumber\\
&+\frac{t^{(k)}}{2}\|\mathbf{w}-\mathbf{w}^{(k)}\|^2+ r(\mathbf{w}),\label{eq:minimizesurrogate}
\end{align}
In fact, problem (\ref{eq:minimizesurrogate}) is equivalent to the following proximal operator problem:
\begin{align}
\mathbf{w}^{(k+1)}=\argmin_{\mathbf{w}}~\frac{1}{2}\|\mathbf{w}-\mathbf{u}^{(k)}\|^2+\frac{1}{t^{(k)}}r(\mathbf{w}),\nonumber
\end{align}
where $\mathbf{u}^{(k)} = \mathbf{w}^{(k)} - \nabla l(\mathbf{w}^{(k)})/t^{(k)}$.
Thus, in GIST we first perform a gradient descent along the direction $-\nabla l(\mathbf{w}^{(k)})$
with step size $1/t^{(k)}$ and then solve a proximal operator problem.
For all the regularizers listed in \TabRef{tab:regularizerexamples}, problem (\ref{eq:minimizesurrogate}) has a closed-form solution
(details are provided in the Appendix), although it may be a non-convex problem.
For example, for the $\ell_1$ and Capped $\ell_1$ regularizers, we have closed-form solutions
as follows:
\begin{align}
&\ell_1: w^{(k+1)}_i=\sign(u^{(k)}_i)\max\left(0,|u^{(k)}_i|-\lambda/t^{(k)}\right),\nonumber\\
&\mathrm{Capped}~\ell_1: w^{(k+1)}_i=\left\{
                                       \begin{array}{ll}
                                         x_1, & \mathrm{if}~h_i(x_1)\leq h_i(x_2), \\
                                         x_2, & \mathrm{otherwise},
                                       \end{array}
                                     \right.\nonumber
\end{align}
where $x_1=\sign(u^{(k)}_i)\max(|u^{(k)}_i|,\theta),~x_2=\sign(u^{(k)}_i)\min(\theta, [|u^{(k)}_i|-\lambda/t^{(k)}]_+)$ and $h_i(x)=0.5(x-u^{(k)}_i)^2+\lambda/t^{(k)}\min(|x|,\theta)$.
The detailed procedure of the GIST algorithm is presented in \AlgRef{alg:gist}. There are two issues that remain to be addressed: how
to initialize $t^{(k)}$ (in Line 4) and how to select a line search criterion (in Line 8) at each outer iteration.

\begin{algorithm}[tb]
   \caption{GIST: General Iterative Shrinkage and Thresholding Algorithm}
   \label{alg:gist}
\footnotesize{
\begin{algorithmic}[1]
   \STATE Choose parameters $\eta>1$ and $t_{\min},t_{\max}$ with $0<t_{\min}<t_{\max}$;
   \STATE Initialize iteration counter $k\leftarrow 0$ and a bounded starting point $\mathbf{w}^{(0)}$;
   \REPEAT
   \STATE $t^{(k)}\in[t_{\min},t_{\max}]$;
   \REPEAT
   \STATE $\mathbf{w}^{(k+1)}\leftarrow\argmin_{\mathbf{w}}~l(\mathbf{w}^{(k)})+\langle\nabla l(\mathbf{w}^{(k)}),\mathbf{w}-\mathbf{w}^{(k)}\rangle
+\frac{t^{(k)}}{2}\|\mathbf{w}-\mathbf{w}^{(k)}\|^2+ r(\mathbf{w})$;
   \STATE $t^{(k)}\leftarrow\eta t^{(k)}$;
   \UNTIL{some line search criterion is satisfied}
   \STATE $k\leftarrow k+1$
   \UNTIL{some stopping criterion is satisfied}
\end{algorithmic}
}
\end{algorithm}\vspace{-0.2cm}

\subsubsection{The Step Size Initialization: $1/t^{(k)}$}
Intuitively, a good step size initialization strategy at each outer iteration can greatly reduce the line search cost (Lines 5-8) and
hence is critical for the fast convergence of the algorithm. In this paper,
we propose to initialize the step size by adopting the Barzilai-Borwein (BB) rule \cite{barzilai1988two},
which uses a diagonal matrix $t^{(k)}I$ to approximate
the Hessian matrix $\nabla^2l(\mathbf{w})$ at $\mathbf{w}=\mathbf{w}^{(k)}$. Denote
\begin{align}
\mathbf{x}^{(k)} = \mathbf{w}^{(k)}-\mathbf{w}^{(k-1)},~\mathbf{y}^{(k)} = \nabla l(\mathbf{w}^{(k)})-\nabla l(\mathbf{w}^{(k-1)}).\nonumber
\end{align}
Then $t^{(k)}$ is initialized at the outer iteration $k$ as
\begin{align}
t^{(k)}=\argmin_{t}\|t\mathbf{x}^{(k)}-\mathbf{y}^{(k)}\|^2=\frac{\langle\mathbf{x}^{(k)},\mathbf{y}^{(k)}\rangle}{\langle\mathbf{x}^{(k)},\mathbf{x}^{(k)}\rangle}.\nonumber
\end{align}

\subsubsection{Line Search Criterion}
One natural and commonly used line search criterion is to require that the objective function value is monotonically decreasing. More specifically,
we propose to accept the step size $1/t^{(k)}$ at the outer iteration $k$ if the following monotone line search criterion is satisfied:
\begin{align}
f(\mathbf{w}^{(k+1)})\leq f(\mathbf{w}^{(k)})-\frac{\sigma}{2}t^{(k)}\|\mathbf{w}^{(k+1)}-\mathbf{w}^{(k)}\|^2,\label{eq:acceptcriterionMonotone}
\end{align}
where $\sigma$ is a constant in the interval $(0,1)$.

A variant of the monotone criterion in \EqRef{eq:acceptcriterionMonotone} is a non-monotone line search criterion \cite{grippo1986nonmonotone,grippo2002nonmonotone,wright2009sparse}.
It possibly accepts the step size $1/t^{(k)}$ even if $\mathbf{w}^{(k+1)}$ yields a larger objective function value than $\mathbf{w}^{(k)}$. Specifically, we propose to accept the step size
$1/t^{(k)}$, if $\mathbf{w}^{(k+1)}$ makes the objective function value smaller than the maximum over previous $m~(m>1)$ iterations, that is,
\begin{align}
f(\mathbf{w}^{(k+1)})&\leq \max_{i=\max(0,k-m+1),\cdots,k}f(\mathbf{w}^{(i)})\nonumber\\
&-\frac{\sigma}{2}t^{(k)}\|\mathbf{w}^{(k+1)}-\mathbf{w}^{(k)}\|^2,\label{eq:acceptcriterionNonmonotone}
\end{align}
where $\sigma\in(0,1)$.

\subsubsection{Convergence Analysis}
Inspired by \citet{wright2009sparse,lu2012iterative}, we present detailed convergence analysis under both monotone and non-monotone line search criteria.
We first present a lemma which guarantees that the monotone line search criterion in \EqRef{eq:acceptcriterionMonotone}
is satisfied. This is a basic support for the convergence of \AlgRef{alg:gist}.

\begin{lemma}\label{lemma:criterionsatisfy}
Let the assumptions \textbf{A1}-\textbf{A3} hold and the constant $\sigma\in (0,1)$ be given.
Then for any integer $k\geq 0$, the monotone line search criterion in \EqRef{eq:acceptcriterionMonotone}
is satisfied whenever $t^{(k)}\geq\beta(l)/(1-\sigma)$.
\end{lemma}
\begin{proof}
Since $\mathbf{w}^{(k+1)}$ is a minimizer of problem (\ref{eq:minimizesurrogate}), we have
\begin{align}
&\langle\nabla l(\mathbf{w}^{(k)}),\mathbf{w}^{(k+1)}-\mathbf{w}^{(k)}\rangle
+\frac{t^{(k)}}{2}\|\mathbf{w}^{(k+1)}-\mathbf{w}^{(k)}\|^2\nonumber\\
&+r(\mathbf{w}^{(k+1)})\leq r(\mathbf{w}^{(k)}).\label{eq:minimumineq}
\end{align}
It follows from assumption \textbf{A1} that
\begin{align}
l(\mathbf{w}^{(k+1)})\leq & l(\mathbf{w}^{(k)}) + \langle\nabla l(\mathbf{w}^{(k)}),\mathbf{w}^{(k+1)}-\mathbf{w}^{(k)}\rangle\nonumber\\
&+\frac{\beta(l)}{2}\|\mathbf{w}^{(k+1)}-\mathbf{w}^{(k)}\|^2.\label{eq:lipschitzineq}
\end{align}
Combining \EqRef{eq:minimumineq} and \EqRef{eq:lipschitzineq}, we have
\begin{align}
&l(\mathbf{w}^{(k+1)})+r(\mathbf{w}^{(k+1)})\leq l(\mathbf{w}^{(k)}) + r(\mathbf{w}^{(k)})\nonumber\\
&-\frac{t^{(k)}-\beta(l)}{2}\|\mathbf{w}^{(k+1)}-\mathbf{w}^{(k)}\|^2.\nonumber
\end{align}
It follows that
\begin{align}
&f(\mathbf{w}^{(k+1)})\leq f(\mathbf{w}^{(k)})-\frac{t^{(k)}-\beta(l)}{2}\|\mathbf{w}^{(k+1)}-\mathbf{w}^{(k)}\|^2.\nonumber
\end{align}
Therefore, the line search criterion in \EqRef{eq:acceptcriterionMonotone} is satisfied whenever
$(t^{(k)}-\beta(l))/2\geq\sigma t^{(k)}/2$, i.e., $t^{(k)}\geq\beta(l)/(1-\sigma)$. This completes
the proof the lemma. \qed
\end{proof}

Next, we summarize the boundedness of $t^{(k)}$ in the following lemma.
\begin{lemma}\label{lemma:tbounded}
For any $k\geq 0$, $t^{(k)}$ is bounded under the monotone line search criterion in \EqRef{eq:acceptcriterionMonotone}.
\end{lemma}
\begin{proof}
It is trivial to show that $t^{(k)}$ is bounded from below, since $t^{(k)}\geq t_{\min}$~($t_{\min}$ is defined in \AlgRef{alg:gist}).
Next we prove that $t^{(k)}$ is bounded from above
by contradiction. Assume that there exists a $k\geq 0$,  such that $t^{(k)}$ is unbounded from above. Without loss
of generality, we assume that  $t^{(k)}$ increases monotonically to $+\infty$ and $t^{(k)}\geq\eta\beta(l)/(1-\sigma)$.
Thus, the value $t=t^{(k)}/\eta\geq\beta(l)/(1-\sigma)$ must have been tried at iteration $k$ and does not satisfy the line search criterion
in \EqRef{eq:acceptcriterionMonotone}. But \LemmaRef{lemma:criterionsatisfy} states that $t=t^{(k)}/\eta\geq\beta(l)/(1-\sigma)$ is
guaranteed to satisfy the line search criterion
in \EqRef{eq:acceptcriterionMonotone}. This leads to a contradiction. Thus, $t^{(k)}$ is bounded from above.\qed
\end{proof}

\begin{remark}\label{remark:tboundedness}
We note that if \EqRef{eq:acceptcriterionMonotone} holds, \EqRef{eq:acceptcriterionNonmonotone} is guaranteed to be satisfied. Thus, the same conclusions in \LemmaRef{lemma:criterionsatisfy}
and \LemmaRef{lemma:tbounded} also hold under the the non-monotone line search criterion in \EqRef{eq:acceptcriterionNonmonotone}.
\end{remark}

Based on \LemmaRef{lemma:criterionsatisfy} and \LemmaRef{lemma:tbounded}, we present our convergence result in the following theorem.
\begin{theorem}\label{theorem:criticalpointMonotone}
Let the assumptions \textbf{A1}-\textbf{A3} hold
and the monotone line search criterion in \EqRef{eq:acceptcriterionMonotone} be satisfied.
Then all limit points of the sequence $\left\{\mathbf{w}^{(k)}\right\}$ generated by \AlgRef{alg:gist} are critical points of problem (\ref{eq:nonconvexopt}).
\end{theorem}
\begin{proof}
Based on \LemmaRef{lemma:criterionsatisfy}, the monotone line search criterion in \EqRef{eq:acceptcriterionMonotone} is satisfied
and hence
\begin{align}
f(\mathbf{w}^{(k+1)})\leq f(\mathbf{w}^{(k)}),\forall k\geq 0,\nonumber
\end{align}
which implies that the sequence $\left\{f(\mathbf{w}^{(k)})\right\}_{k=0,1,\cdots}$ is monotonically decreasing.
Let $\mathbf{w}^\star$ be a limit point of the sequence $\left\{\mathbf{w}^{(k)}\right\}$, that is,
there exists a subsequence $\mathcal{K}$ such that
\begin{align}
\lim_{k\in\mathcal{K}\rightarrow\infty}\mathbf{\mathbf{w}}^{(k)}=\mathbf{w}^\star.\nonumber
\end{align}
Since $f$ is bounded from below, together with the fact that $\left\{f(\mathbf{w}^{(k)})\right\}$ is monotonically decreasing,
$\lim_{k\rightarrow\infty}f(\mathbf{w}^{(k)})$ exists. Observing that $f$ is continuous, we have
\begin{align}
\lim_{k\rightarrow\infty}f(\mathbf{w}^{(k)})=\lim_{k\in\mathcal{K}\rightarrow\infty}f(\mathbf{w}^{(k)})=f(\mathbf{w}^\star).\nonumber
\end{align}
Taking limits on both sides of \EqRef{eq:acceptcriterionMonotone} with $k\in\mathcal{K}$, we have
\begin{align}
\lim_{k\in\mathcal{K}\rightarrow\infty}\|\mathbf{w}^{(k+1)}-\mathbf{w}^{(k)}\|=0.\label{eq:dwkzero}
\end{align}
Considering that the minimizer $\mathbf{w}^{(k+1)}$ is also a critical point of problem (\ref{eq:minimizesurrogate})
and $r(\mathbf{w})=r_1(\mathbf{w})-r_2(\mathbf{w})$, we have
\begin{align}
\mathbf{0}\in &\nabla l(\mathbf{w}^{(k)})+t^{(k)}(\mathbf{w}^{(k+1)}-\mathbf{w}^{(k)})\nonumber\\
&+\partial r_1(\mathbf{w}^{(k+1)})-\partial r_2(\mathbf{w}^{(k+1)}).\nonumber
\end{align}
Taking limits on both sides of the above equation with $k\in\mathcal{K}$, by considering
the semi-continuity of $\partial r_1(\cdot)$ and $\partial r_2(\cdot)$,
the boundedness of $t^{(k)}$ (based on \LemmaRef{lemma:tbounded}) and \EqRef{eq:dwkzero}, we obtain
\begin{align}
\mathbf{0}\in\nabla l(\mathbf{w}^\star)+\partial r_1(\mathbf{w}^\star)-\partial r_2(\mathbf{w}^\star),\nonumber
\end{align}
Therefore, $\mathbf{w}^\star$ is a critical point of problem (\ref{eq:nonconvexopt}). This completes the proof of \ThemRef{theorem:criticalpointMonotone}.\qed
\end{proof}

Based on \EqRef{eq:dwkzero}, we know that $\lim_{k\in\mathcal{K}\rightarrow\infty}\|\mathbf{w}^{(k+1)}-\mathbf{w}^{(k)}\|^2=0$ is a
necessary optimality condition of \AlgRef{alg:gist}. Thus, $\|\mathbf{w}^{(k+1)}-\mathbf{w}^{(k)}\|^2$ is a quantity to measure
the convergence of the sequence $\{\mathbf{w}^{(k)}\}$ to a critical point. We present the convergence rate in terms of $\|\mathbf{w}^{(k+1)}-\mathbf{w}^{(k)}\|^2$ in the following theorem.
\begin{theorem}\label{theorem:convergencerate}
Let $\{\mathbf{w}^{(k)}\}$ be the sequence generated by \AlgRef{alg:gist} with the monotone line search criterion in \EqRef{eq:acceptcriterionMonotone} satisfied. Then for every $n\geq 1$, we have
\begin{align}
\min_{0\leq k \leq n}\|\mathbf{w}^{(k+1)}-\mathbf{w}^{(k)}\|^2\leq\frac{2(f(\mathbf{w}^{(0)})-f(\mathbf{w}^{\star}))}{n\sigma t_{\min}},\nonumber
\end{align}
where $\mathbf{w}^{\star}$ is a limit point of the sequence $\{\mathbf{w}^{(k)}\}$.
\end{theorem}
\begin{proof}
Based on \EqRef{eq:acceptcriterionMonotone} with $t^{(k)}\geq t_{\min}$, we have
\begin{align}
\frac{\sigma t_{\min}}{2}\|\mathbf{w}^{(k+1)}-\mathbf{w}^{(k)}\|^2\leq f(\mathbf{w}^{(k)})-f(\mathbf{w}^{(k+1)}).\nonumber
\end{align}
Summing the above inequality over $k=0,\cdots,n$, we obtain
\begin{align}
\frac{\sigma t_{\min}}{2}\sum_{k=0}^{n}\|\mathbf{w}^{(k+1)}-\mathbf{w}^{(k)}\|^2\leq f(\mathbf{w}^{(0)})-f(\mathbf{w}^{(n+1)}),\nonumber
\end{align}
which implies that
\begin{align}
\min_{0\leq k \leq n}\|\mathbf{w}^{(k+1)}-\mathbf{w}^{(k)}\|^2&\leq\frac{2(f(\mathbf{w}^{(0)})-f(\mathbf{w}^{(n+1)}))}{n\sigma t_{\min}}\nonumber\\
&\leq\frac{2(f(\mathbf{w}^{(0)})-f(\mathbf{w}^{\star}))}{n\sigma t_{\min}}.\nonumber
\end{align}
This completes the proof of the theorem.\qed
\end{proof}

Under the non-monotone line search criterion in \EqRef{eq:acceptcriterionNonmonotone}, we have a similar convergence result in the following theorem
(the proof uses an extension of argument for \ThemRef{theorem:criticalpointMonotone} and is omitted).
\begin{theorem}\label{theorem:criticalpointNonmonotone}
Let the assumptions \textbf{A1}-\textbf{A3} hold
and the non-monotone line search criterion in \EqRef{eq:acceptcriterionNonmonotone} be satisfied.
Then all limit points of the sequence $\left\{\mathbf{w}^{(k)}\right\}$ generated by \AlgRef{alg:gist} are critical points of problem (\ref{eq:nonconvexopt}).
\end{theorem}

Note that \ThemRef{theorem:criticalpointMonotone}/\ThemRef{theorem:criticalpointNonmonotone}
makes sense only if $\left\{\mathbf{w}^{(k)}\right\}$ has limit points. By considering
one more mild assumption:
\begin{itemize}
\item[\textbf{A4}] $f(\mathbf{w})\rightarrow+\infty$ when $\|\mathbf{w}\|\rightarrow+\infty$,
\end{itemize}
we summarize the existence of limit points in the following theorem (the proof is omitted):
\begin{theorem}\label{theorem:limitpointsexist}
Let the assumptions \textbf{A1}-\textbf{A4} hold and the monotone/non-monotone line search criterion in
\EqRef{eq:acceptcriterionMonotone}/\EqRef{eq:acceptcriterionNonmonotone} be satisfied.
Then the sequence $\left\{\mathbf{w}^{(k)}\right\}$ generated by \AlgRef{alg:gist} has at least one limit point.
\end{theorem}

\subsubsection{Discussions}
Observe that $l(\mathbf{w}^{(k)})+\langle\nabla l(\mathbf{w}^{(k)}),\mathbf{w}-\mathbf{w}^{(k)}\rangle+\frac{t^{(k)}}{2}\|\mathbf{w}-\mathbf{w}^{(k)}\|^2$
can be viewed as an approximation of $l(\mathbf{w})$ at $\mathbf{w}=\mathbf{w}^{(k)}$.
The GIST algorithm minimizes an approximate surrogate instead of the objective
function in problem (\ref{eq:nonconvexopt}) at each outer iteration. We further observe that if $t^{(k)}\geq\beta(l)/(1-\sigma)>\beta(l)$
[the sufficient condition of \EqRef{eq:acceptcriterionMonotone}], we obtain
\begin{align}
l(\mathbf{w})\leq & l(\mathbf{w}^{(k)}) + \langle\nabla l(\mathbf{w}^{(k)}),\mathbf{w}-\mathbf{w}^{(k)}\rangle\nonumber\\
&+\frac{t^{(k)}}{2}\|\mathbf{w}-\mathbf{w}^{(k)}\|^2,\forall \mathbf{w}\in\mathbb{R}^d.\nonumber
\end{align}
It follows that
\begin{align}
f(\mathbf{w})=l(\mathbf{w})+r(\mathbf{w})\leq M(\mathbf{w},\mathbf{w}^{(k)}),\forall \mathbf{w}\in\mathbb{R}^d,\nonumber
\end{align}
where $M(\mathbf{w},\mathbf{w}^{(k)})$ denotes the objective function of problem (\ref{eq:minimizesurrogate}). We can easily show that
\begin{align}
f(\mathbf{w}^{(k)})= M(\mathbf{w}^{(k)},\mathbf{w}^{(k)}).\nonumber
\end{align}
Thus, the GIST algorithm is equivalent to solving a sequence of minimization problems:
\begin{align}
\mathbf{w}^{(k+1)}=\argmin_{\mathbf{w}}M(\mathbf{w},\mathbf{w}^{(k)}), ~k=0,1,2,\cdots\nonumber
\end{align}
and can be interpreted as the well-known Majorization and Minimization (MM) technique \cite{hunter2000quantile}.

Note that we focus on the vector case in this paper and the proposed GIST algorithm can be easily extended to the matrix case.

\section{Related Work}\label{sec:relatedwork}
In this section, we discuss some related algorithms. One commonly used approach to solve problem (\ref{eq:nonconvexopt}) is the Multi-Stage (MS)
convex relaxation (or CCCP, or DC programming) \cite{zhang2010analysis,yuille2003concave,gasso2009recovering}. It equivalently rewrites problem (\ref{eq:nonconvexopt}) as
\begin{align}
\min_{\mathbf{w}\in\mathbb{R}^d}f_1(\mathbf{w}) - f_2(\mathbf{w}),\nonumber
\end{align}
where $f_1(\mathbf{w})$ and $f_2(\mathbf{w})$ are both convex functions. The MS algorithm solves problem (\ref{eq:nonconvexopt}) by generating a sequence $\{\mathbf{w}^{(k)}\}$ as
\begin{align}
\mathbf{w}^{(k+1)} = &\argmin_{\mathbf{w}\in\mathbb{R}^d}f_1(\mathbf{w}) - f_2(\mathbf{w}^{(k)}) \nonumber\\
&-\langle \mathbf{s}_2(\mathbf{w}^{(k)}),\mathbf{w}- \mathbf{w}^{(k)}\rangle,\label{eq:dcp}
\end{align}
where $\mathbf{s}_2(\mathbf{w}^{(k)})$ denotes a sub-gradient of $f_2(\mathbf{w})$ at $\mathbf{w}=\mathbf{w}^{(k)}$.
Obviously, the objective function in problem (\ref{eq:dcp}) is convex. The MS algorithm involves solving
a sequence of convex optimization problems as in problem (\ref{eq:dcp}). In general, there is no closed-form solution to problem (\ref{eq:dcp})
and the computational cost of the MS algorithm is $k$ times that of solving problem (\ref{eq:dcp}), where $k$ is the number
of outer iterations. This is computationally expensive especially for large scale problems.

A class of related algorithms called iterative shrinkage and thresholding (IST), which are also known as different names such as fixed point iteration and
forward-backward splitting \cite{daubechies2004iterative,combettes2005signal,hale2007fixed,beck2009fast,wright2009sparse,Liu:2009:SLEP:manual}, have been extensively
applied to solve problem (\ref{eq:nonconvexopt}). The key step is by generating a sequence $\{\mathbf{w}^{(k)}\}$ via solving problem (\ref{eq:minimizesurrogate}).
However, they require that the regularizer $r(\mathbf{w})$ is \emph{convex} and some of them even require that both $l(\mathbf{w})$
and $r(\mathbf{w})$ are convex. Our proposed GIST algorithm is
a more general framework, which can deal with a wider range of problems including both convex and non-convex cases.

Another related algorithm called a Variant of Iterative Reweighted $L_\alpha$ (VIRL) is recently proposed to
solve the following optimization problem \cite{lu2012iterative}:
\begin{align}
\min_{\mathbf{w}\in\mathbb{R}^d}\left\{f(\mathbf{w})=l(\mathbf{w})+\lambda\sum_{i=1}^d(|w_i|^\alpha+\epsilon_i)^{q/\alpha}\right\}, \nonumber
\end{align}
where $\alpha\geq 1,0<q<1, \epsilon_i>0$. VIRL solves the above problem by generating a sequence $\{\mathbf{w}^{(k)}\}$ as
\begin{align}
&\mathbf{w}^{(k+1)} = \argmin_{\mathbf{w}\in\mathbb{R}^d}l(\mathbf{w}^{(k)}) + \langle \nabla l(\mathbf{w}^{(k)}),\mathbf{w}- \mathbf{w}^{(k)}\rangle \nonumber\\
&+\frac{t^{(k)}}{2}\|\mathbf{w}- \mathbf{w}^{(k)}\|^2+\frac{\lambda q}{\alpha}\sum_{i=1}^d(|w_i^k|^\alpha+\epsilon_i)^{q/\alpha-1}|w_i|^\alpha.\nonumber
\end{align}
In VIRL, $t^{(k-1)}$ is chosen as the initialization of $t^{(k)}$. The line search
step in VIRL finds the smallest integer $\ell$ with $t^{(k)}=t^{(k-1)}\eta^{\ell}~(\eta>1)$ such that
\begin{align}
f(\mathbf{w}^{(k+1)})\leq f(\mathbf{w}^{(k)})-\frac{\sigma}{2}\|\mathbf{w}^{(k+1)}- \mathbf{w}^{(k)}\|^2~(\sigma>0).\nonumber
\end{align}

The most related algorithm to our propose GIST is the Sequential Convex Programming (SCP) proposed by \citet{lu2012sequential}. SCP
solves problem (\ref{eq:nonconvexopt}) by generating a sequence $\{\mathbf{w}^{(k)}\}$ as
\begin{align}
&\mathbf{w}^{(k+1)} = \argmin_{\mathbf{w}\in\mathbb{R}^d}l(\mathbf{w}^{(k)}) + \langle \nabla l(\mathbf{w}^{(k)}),\mathbf{w}- \mathbf{w}^{(k)}\rangle \nonumber\\
&+\frac{t^{(k)}}{2}\|\mathbf{w}- \mathbf{w}^{(k)}\|^2+r_1(\mathbf{w})-r_2(\mathbf{w}^{(k)})-\langle\mathbf{s}_2,\mathbf{w}-\mathbf{w}^{(k)}\rangle, \nonumber
\end{align}
where $\mathbf{s}_2$ is a sub-gradient of $r_2(\mathbf{w})$ at $\mathbf{w}=\mathbf{w}^{(k)}$. Our algorithm differs from SCP in
that the original regularizer $r(\mathbf{w})=r_1(\mathbf{w})-r_2(\mathbf{w})$ is used in the proximal operator in problem (\ref{eq:minimizesurrogate}), while $r_1(\mathbf{w})$ minus
a locally linear approximation for $r_2(\mathbf{w})$ is adopted in SCP. We will show in the experiments that our proposed GIST algorithm is more efficient than SCP.

\section{Experiments}\label{sec:experiments}
\subsection{Experimental Setup}
We evaluate our GIST algorithm by considering the Capped $\ell_1$ regularized logistic regression problem, that is $l(\mathbf{w})=\frac{1}{n}\sum_{i=1}^n\log\left(1+\exp(-y_i\mathbf{x}_i^T\mathbf{w})\right)$ and $r(\mathbf{w})=\lambda\sum_{i=1}^d\min(|w_i|,\theta)$.
We compare our GIST algorithm with the Multi-Stage (MS) algorithm
and the SCP algorithm in different settings using twelve data sets summarized in \TabRef{tab:textdataset}. These data
sets are high dimensional and sparse. Two of them (news20, real-sim)\footnote{http://www.csie.ntu.edu.tw/cjlin/libsvmtools/datasets/}
have been preprocessed as two-class data sets \cite{lin2008trust}. The other ten\footnote{http://www.shi-zhong.com/software/docdata.zip}
are multi-class data sets. We transform the multi-class data sets into two-class by labeling the first half of all classes as
positive class, and the remaining classes as the negative class.

All algorithms are implemented in Matlab and executed on an Intel(R) Core(TM)2 Quad CPU (Q6600 @2.4GHz) with 8GB memory.
We set $\sigma=10^{-5},m=5,\eta=2,1/t_{\min}=t_{\max}=10^{30}$ and choose the starting points $\mathbf{w}^{(0)}$ of all algorithms
as zero vectors. We terminate all algorithms if the relative change of the two consecutive objective function values
is less than $10^{-5}$ or the number of iterations exceeds $1000$. The Matlab codes of the GIST algorithm are available online~\cite{gong2013gist}.

\subsection{Experimental Evaluation and Analysis}
We report the objective function value vs. CPU time plots with
different parameter settings in \FigRef{fig:objvstime}. From these figures, we have the following observations: (1) Both GISTbb-Monotone and GISTbb-Nonmonotone
decrease the objective function value rapidly and they always have the fastest convergence speed, which shows that adopting the BB rule to initialize $t^{(k)}$
indeed greatly accelerates the convergence speed. Moreover, both GISTbb-Monotone and GISTbb-Nonmonotone algorithms
achieve the smallest objective function values. (2) GISTbb-Nonmonotone may give rise to an increasing objective function value but finally converges and has
a faster overall convergence speed than GISTbb-Monotone in most cases, which indicates that the non-monotone line search
criterion can further accelerate the convergence speed. (3) SCPbb-Nonmonotone is comparable to GISTbb-Nonmonotone in several cases, however, it
converges much slower and achieves much larger objective function values than those of GISTbb-Nonmonotone in the remaining cases. This demonstrates
the superiority of using the original regularizer $r(\mathbf{w})=r_1(\mathbf{w})-r_2(\mathbf{w})$ in the proximal operator in problem (\ref{eq:minimizesurrogate}).
(4) GIST-1 has a faster convergence speed than GIST-$t^{(k-1)}$ in most cases, which
demonstrates that it is a bad strategy to use $t^{(k-1)}$ to initialize $t^{(k)}$. This is because $\{t^{(k)}\}$ increases monotonically in this way, making
the step size $1/t^{(k)}$ monotonically decreasing when the algorithm proceeds.

\begin{table*}[!ht]\vspace{-0.0cm}
\caption{Data sets statistics: $n$ is the number of samples and $d$ is the dimensionality of
the data.} \vspace{-0.0cm} \label{tab:textdataset}
\begin{center}
\scriptsize{
\begin{tabular}{c||cccccccccccc}
\hline\hline
No. &1 &2 &3 &4 &5 &6 &7 &8 &9 &10 &11 &12\\
\hline
datasets &classic &hitech &k1b &la12 &la1 &la2 &news20 &ng3sim &ohscal &real-sim &reviews &sports \\
\hline
$n$ &7094 &2301 &2340 &2301 &3204 &3075 &19996 &2998 &11162 &72309 &4069 &8580   \\
$d$ &41681 &10080 &21839 &31472 &31472 &31472 &1355191 &15810 &11465 &20958 &18482 &14866 \\
\hline\hline
\end{tabular}}
\end{center}\vspace{-0.0cm}
\end{table*}

\begin{figure*}[!ht]\vspace{-0.3cm}

\begin{minipage}[c]{1.0\linewidth}
\centering
\includegraphics[width=.23\linewidth]{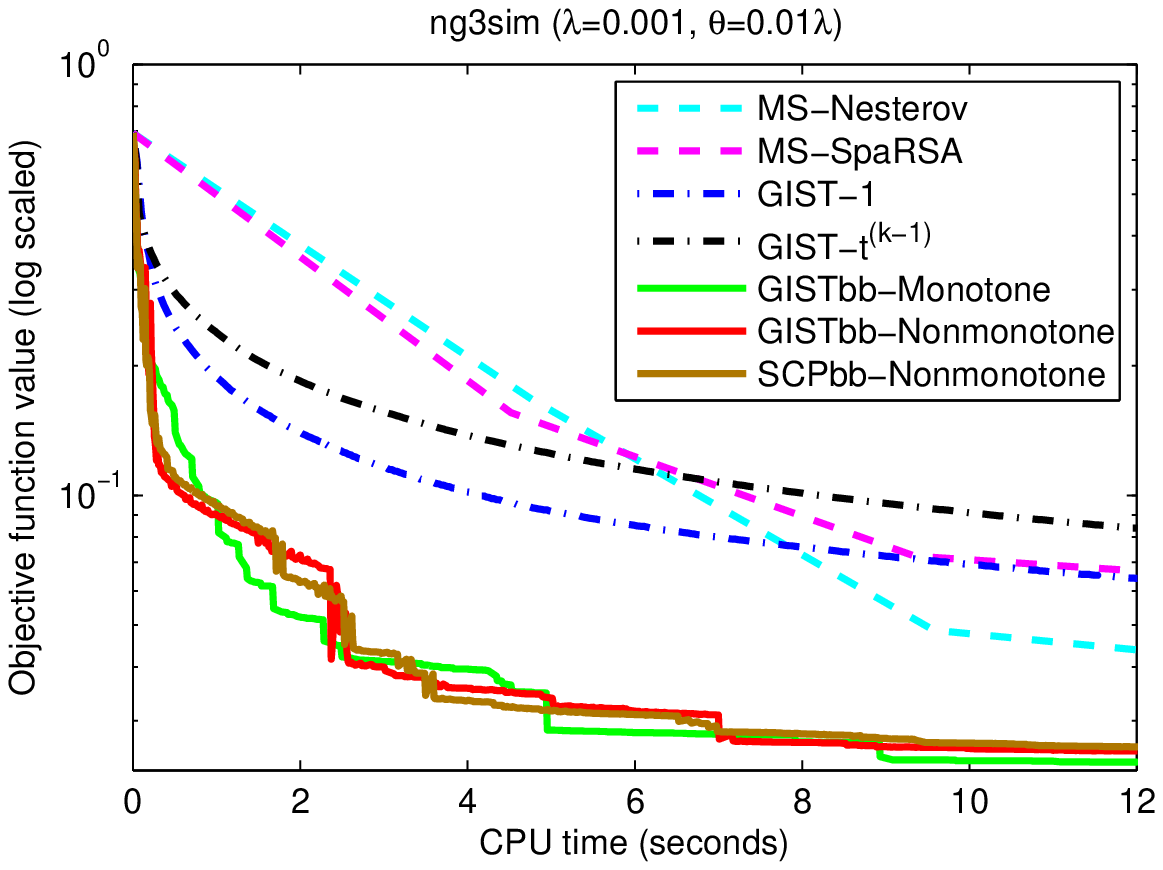}
\includegraphics[width=.23\linewidth]{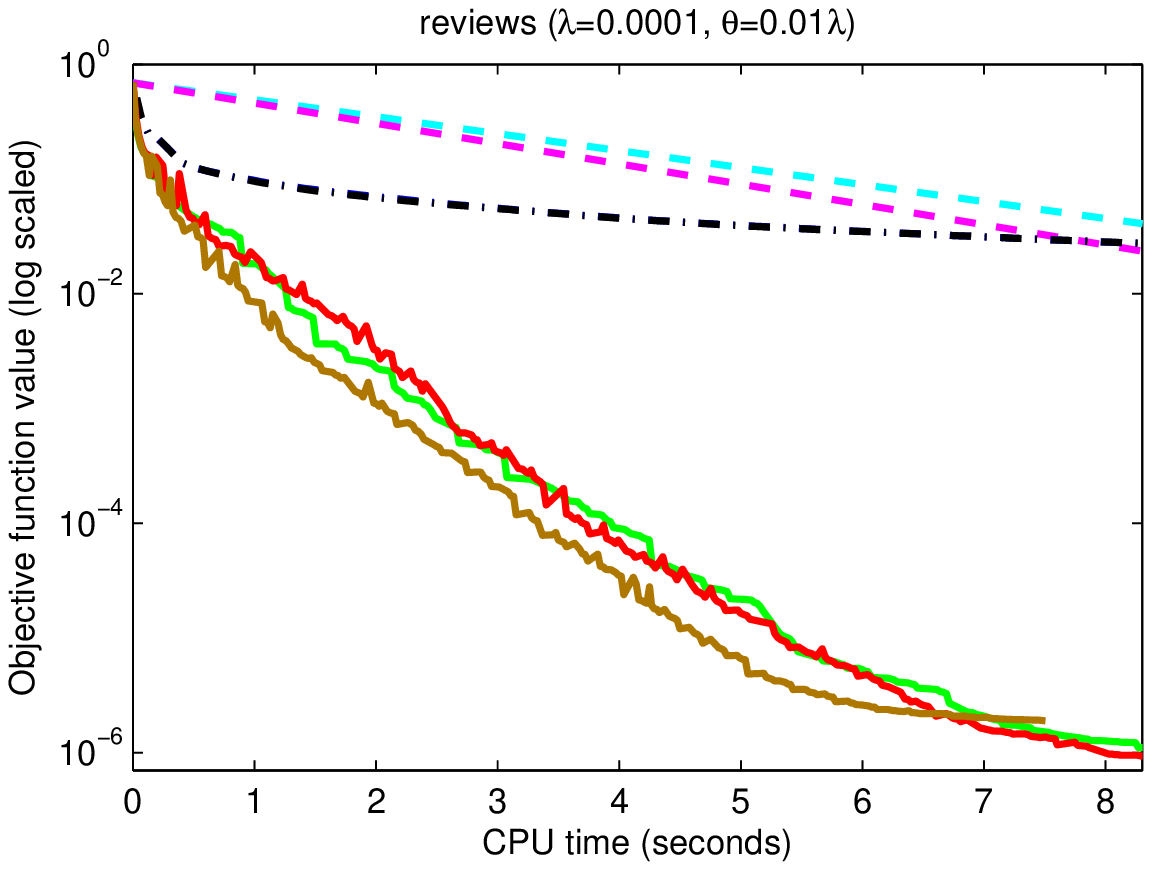}
\includegraphics[width=.23\linewidth]{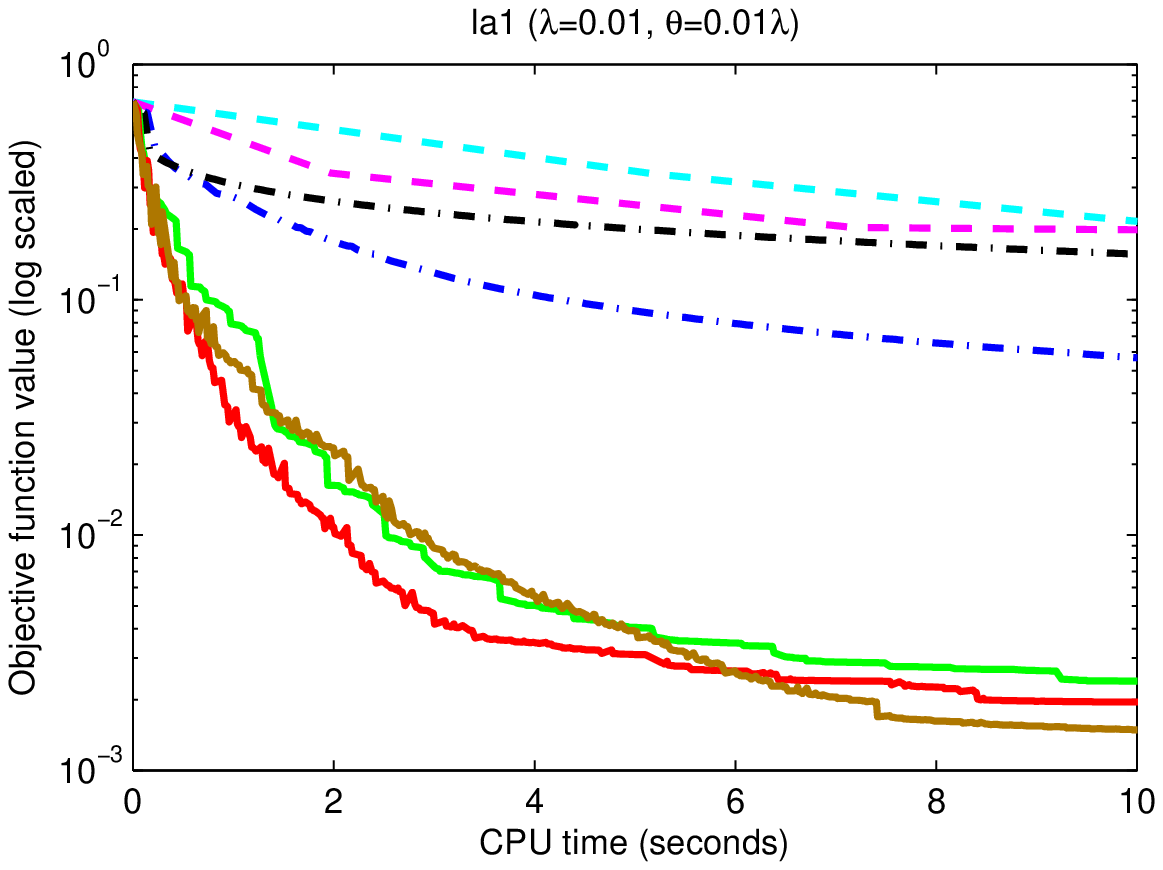}
\includegraphics[width=.23\linewidth]{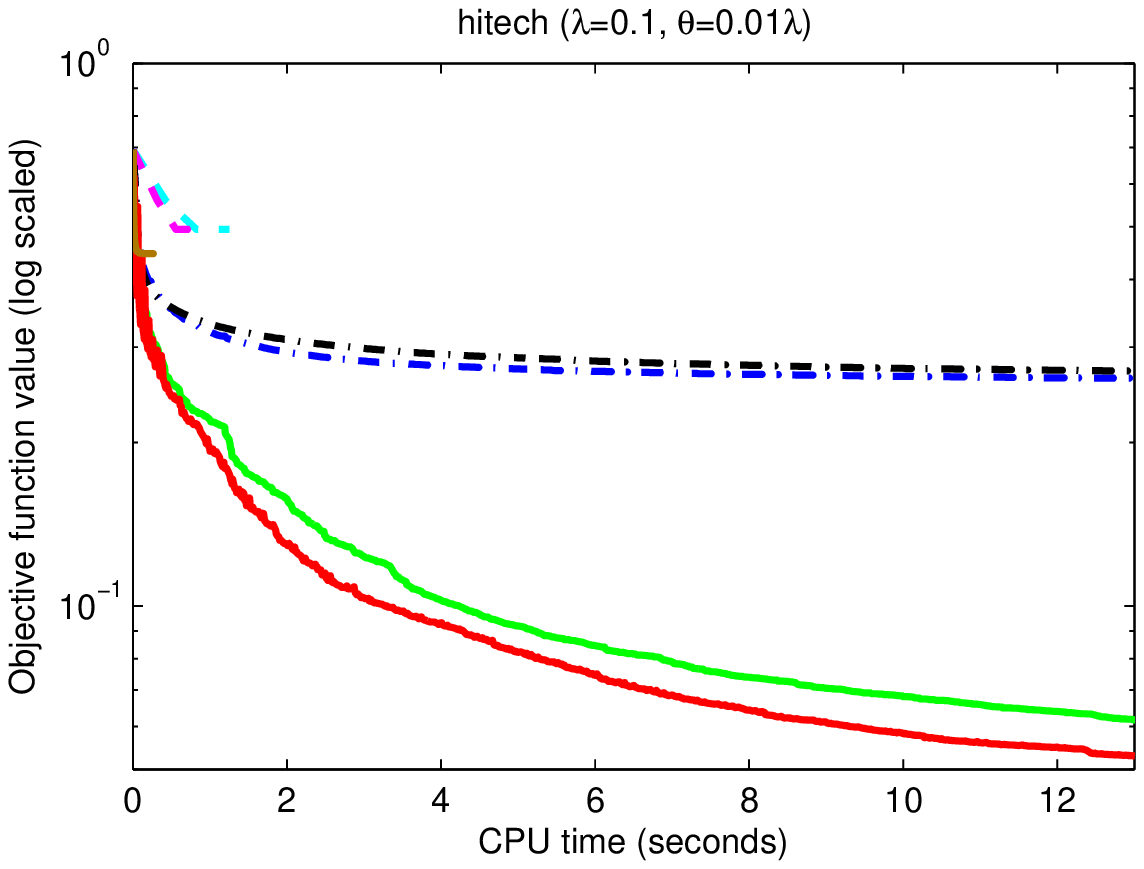}
\end{minipage}
\vskip 0.0cm
\begin{minipage}[c]{1.0\linewidth}
\centering
\includegraphics[width=.23\linewidth]{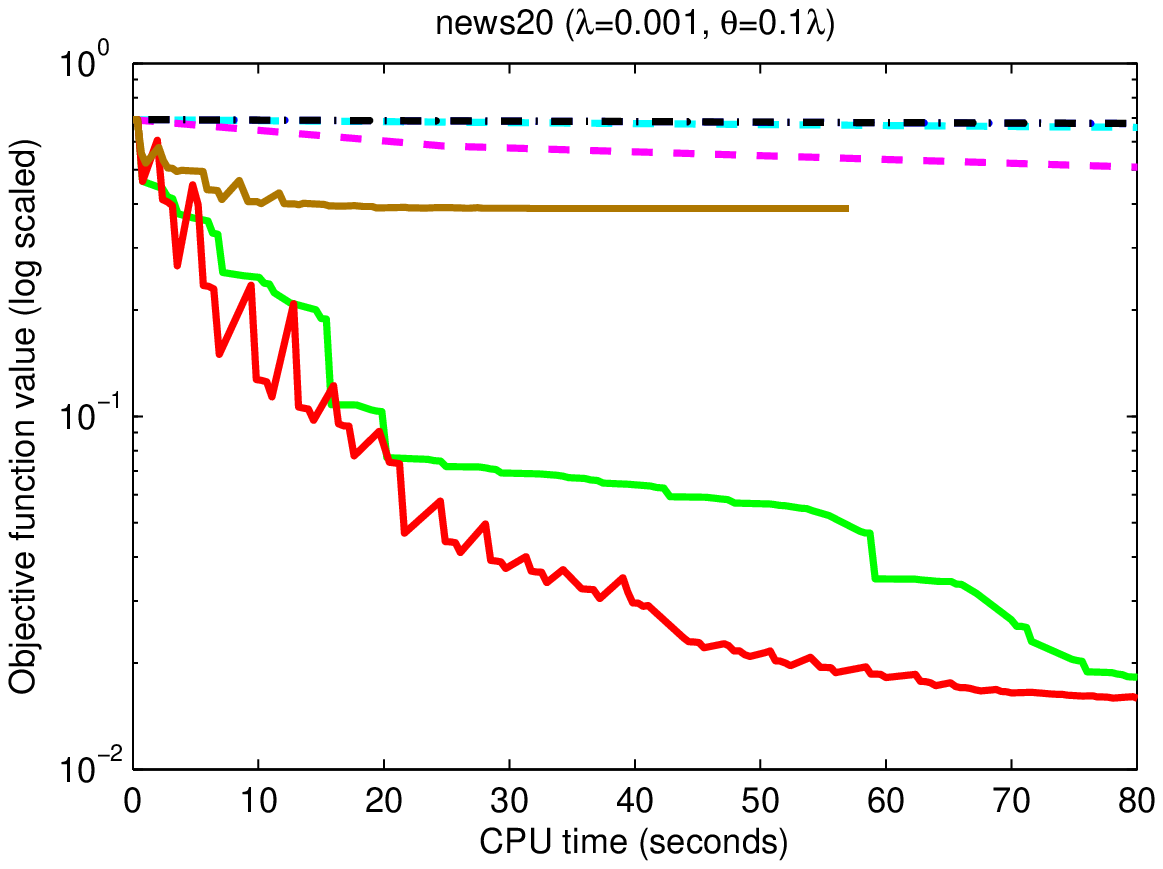}
\includegraphics[width=.23\linewidth]{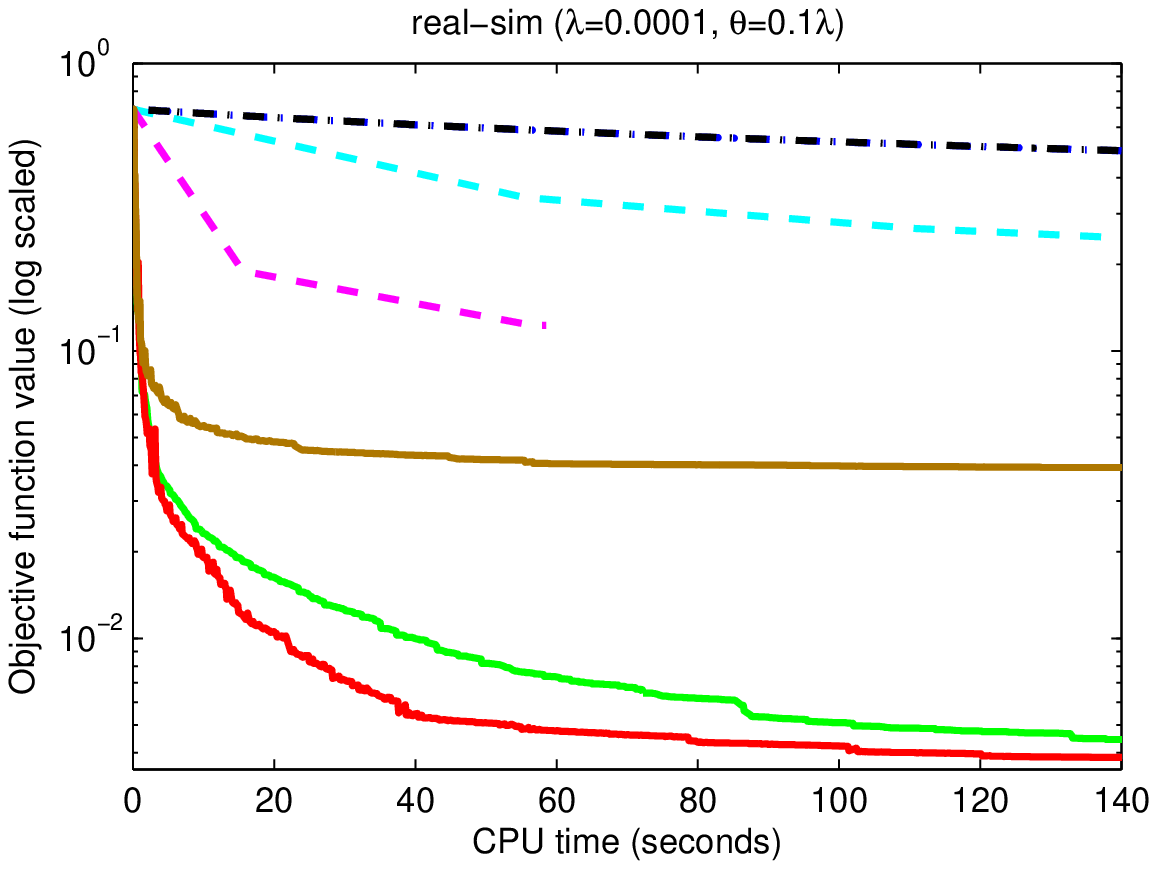}
\includegraphics[width=.23\linewidth]{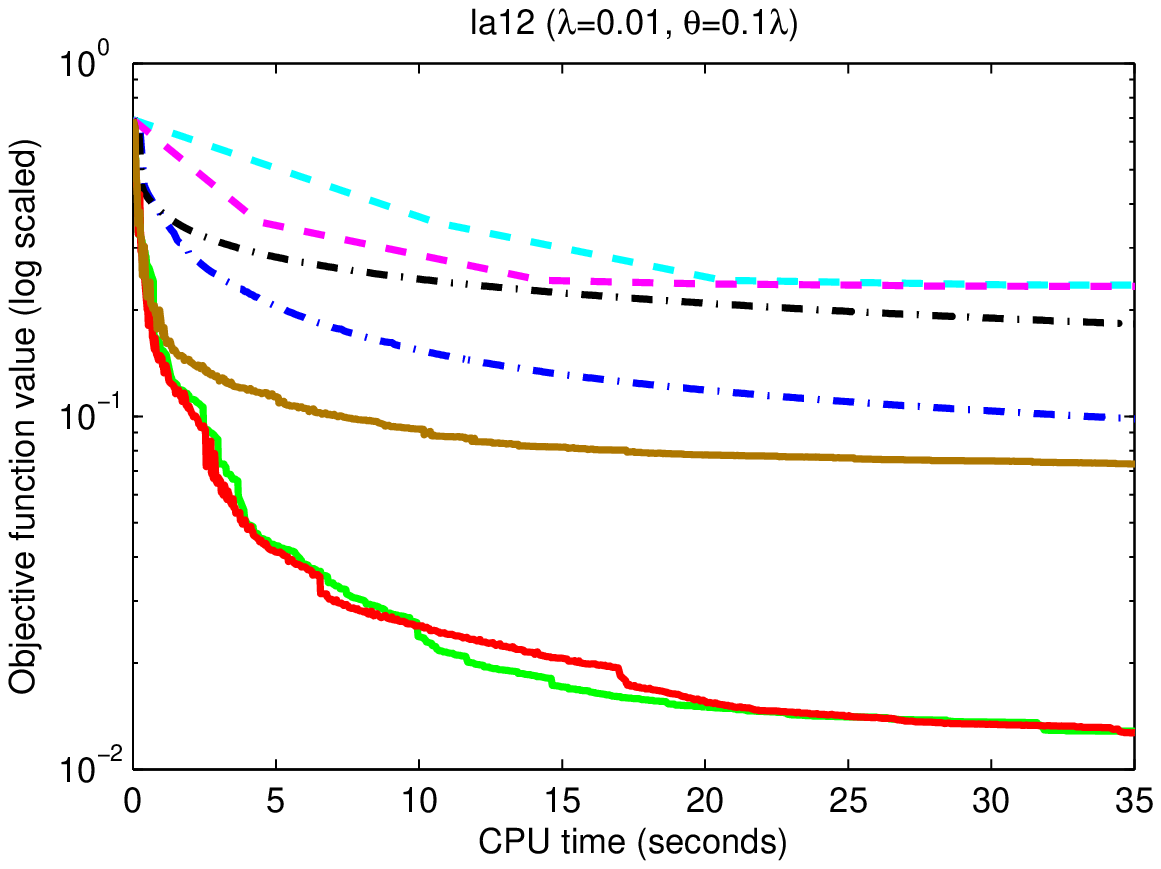}
\includegraphics[width=.23\linewidth]{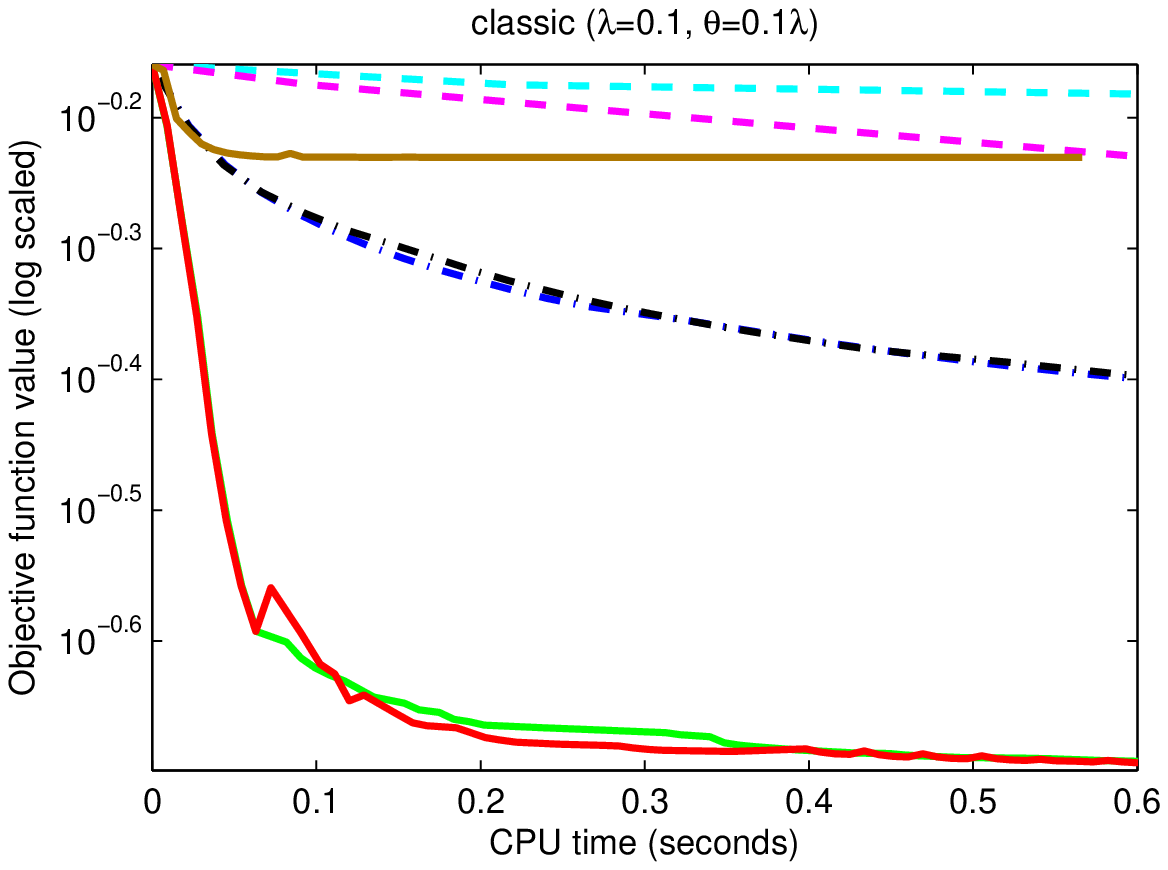}
\end{minipage}
\vskip 0.0cm
\begin{minipage}[c]{1.0\linewidth}
\centering
\includegraphics[width=.23\linewidth]{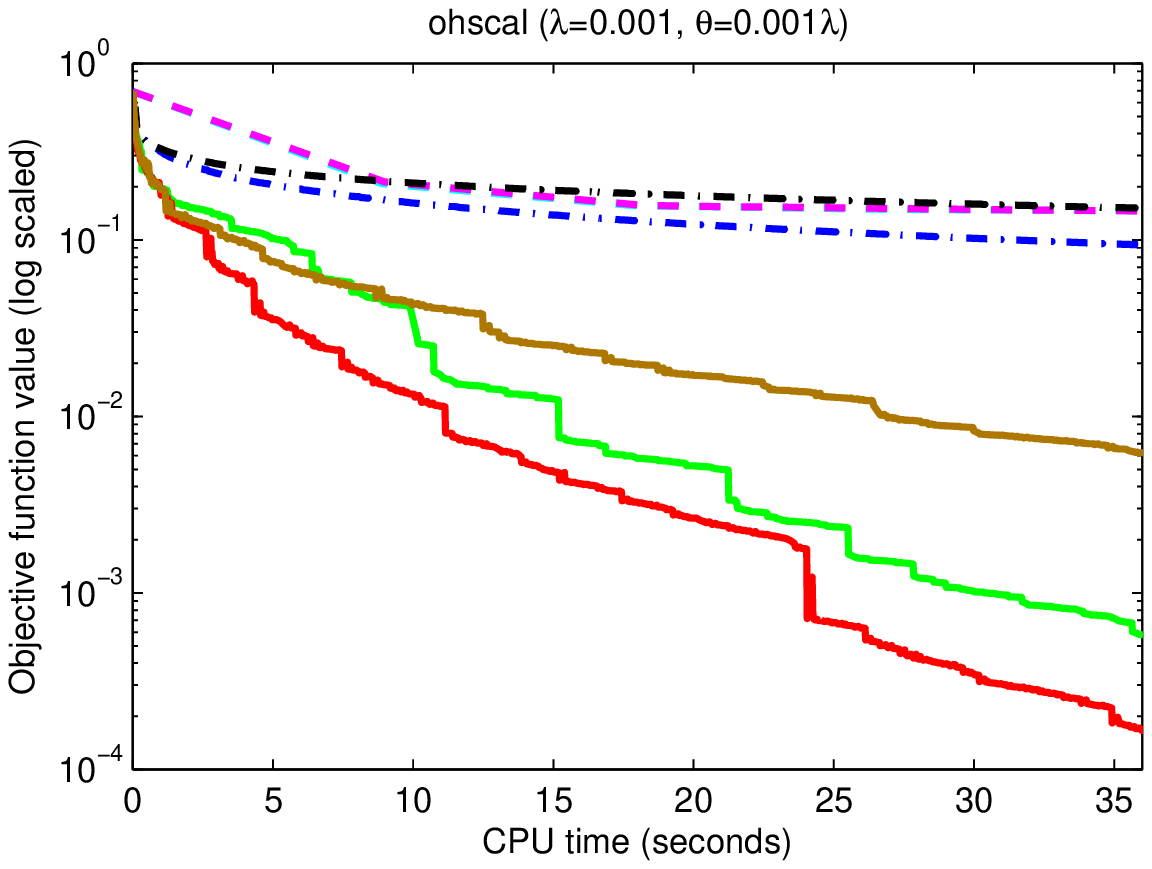}
\includegraphics[width=.23\linewidth]{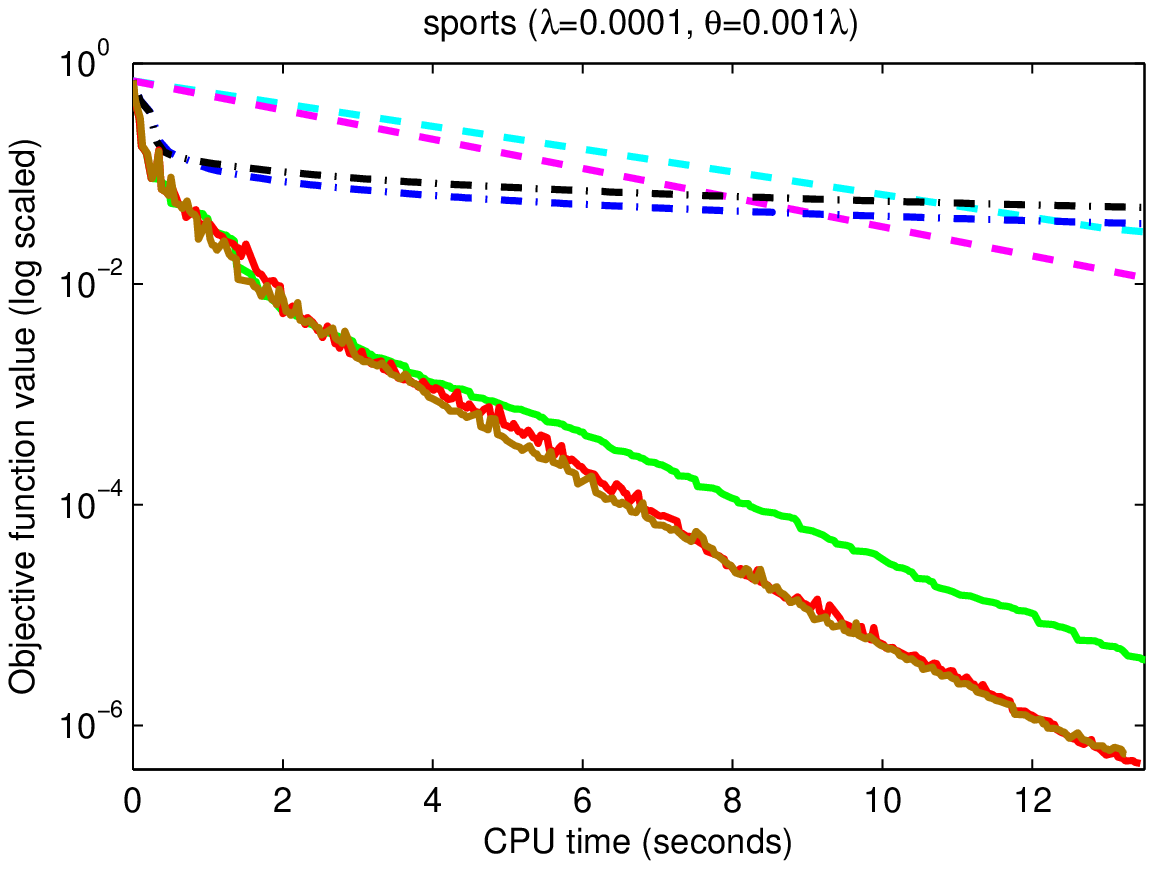}
\includegraphics[width=.23\linewidth]{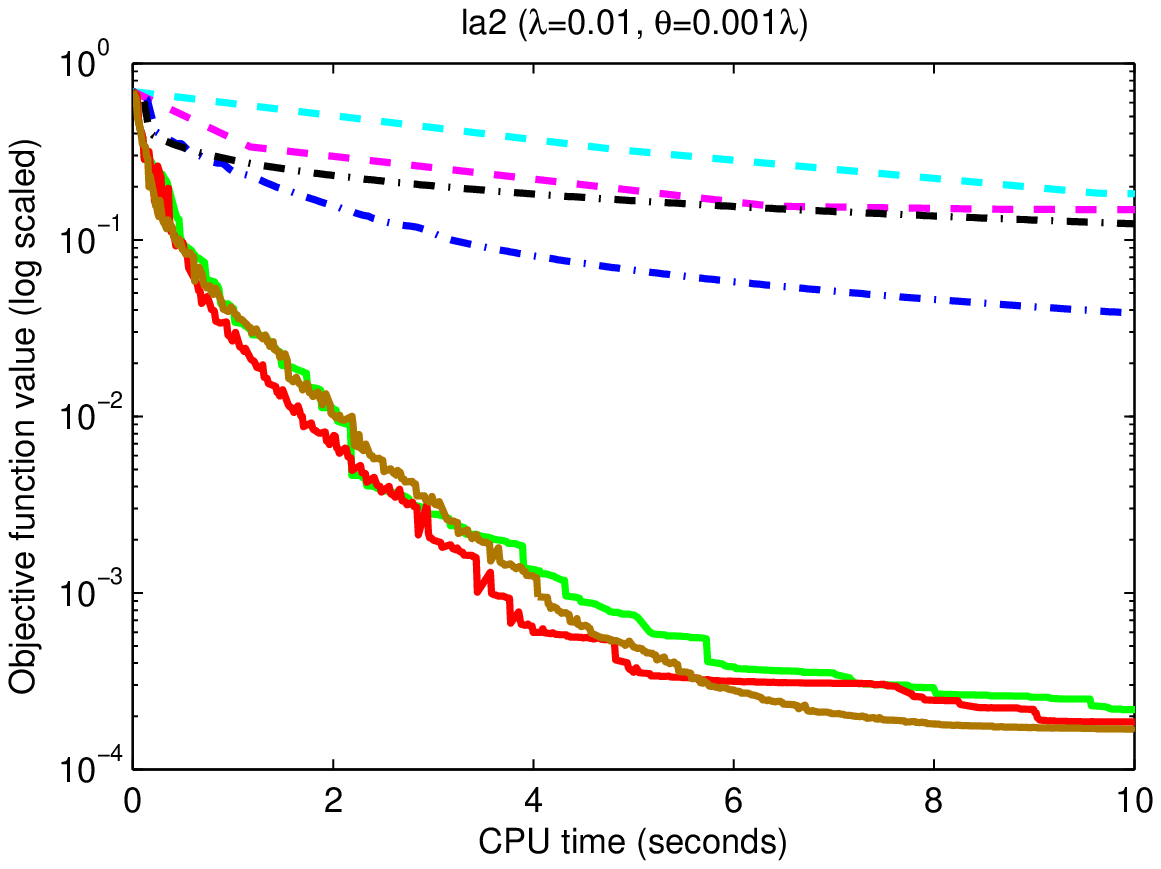}
\includegraphics[width=.23\linewidth]{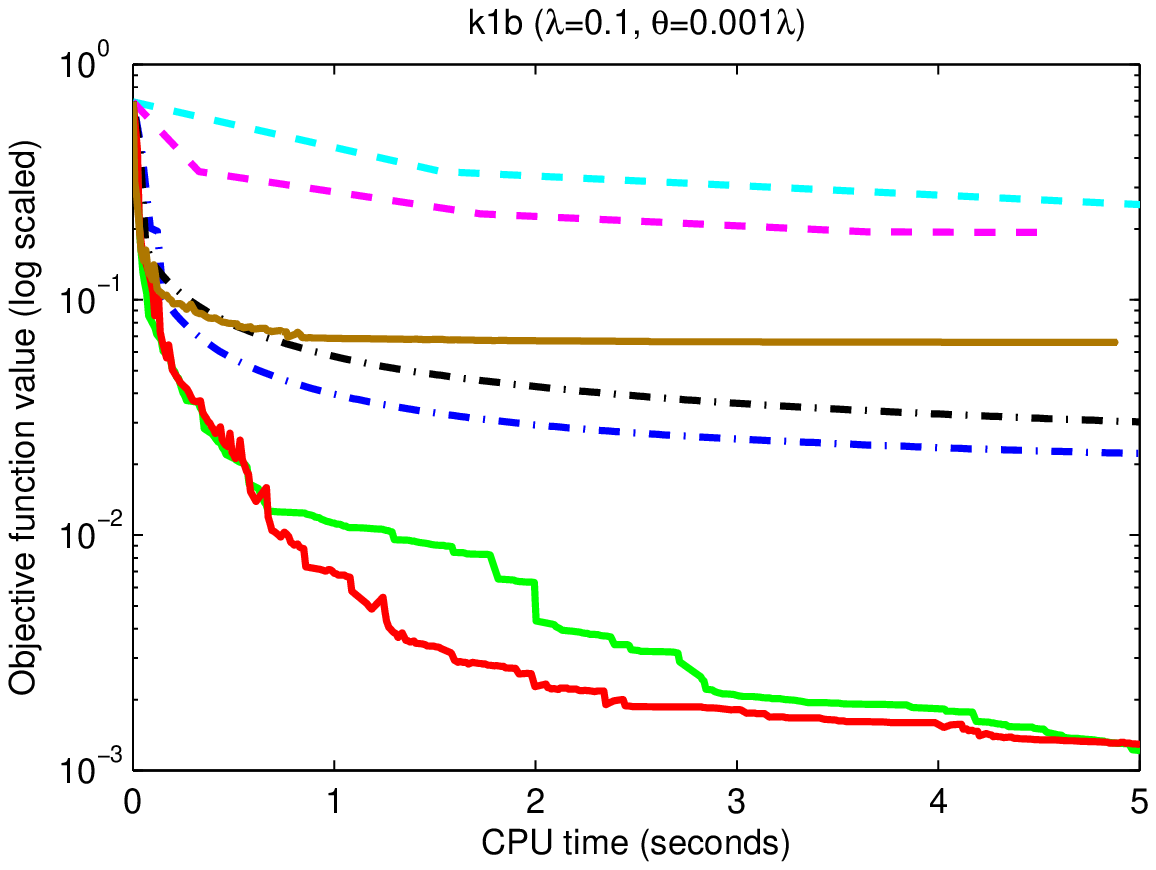}
\end{minipage}
\vspace{-0.3cm}\caption{Objective function value vs. CPU time plots. MS-Nesterov/MS-SpaRSA: The Multi-Stage algorithm using the Nesterov/SpaRSA method to solve problem (\ref{eq:dcp}); GIST-1/GIST-$t^{(k-1)}$/GISTbb-Monotone/GISTbb-Nonmonotone: The GIST algorithm using $1$/$t^{(k-1)}$/BB rule/BB rule to initialize $t^{(k)}$ and
\EqRef{eq:acceptcriterionMonotone}/\EqRef{eq:acceptcriterionMonotone}/\EqRef{eq:acceptcriterionMonotone}/\EqRef{eq:acceptcriterionNonmonotone} as the line search criterion;
SCPbb-Nonmonotone: The SCP algorithm using the BB rule to initialize $t^{(k)}$ and \EqRef{eq:acceptcriterionNonmonotone} as the line search criterion.
Note that on data sets `hitech' and `real-sim', MS algorithms stop early (the SCP algorithm has similar behaviors on data sets `hitech' and `news20'),
because they satisfy the termination condition that the relative change of the two consecutive objective function values
is less than $10^{-5}$. However, their objective function values are much larger than those of GISTbb-Monotone and GISTbb-Nonmonotone.}
\label{fig:objvstime}\vspace{-0.4cm}
\end{figure*}

\section{Conclusions}\label{sec:conclusions}
We propose an efficient iterative shrinkage and thresholding algorithm to solve a general class of non-convex optimization problems
encountered in sparse learning. A critical step of the proposed
algorithm is the computation of a proximal operator, which has a closed-form solution for many commonly used formulations. We propose
to initialize the step size at each iteration using the BB rule and employ both monotone and non-monotone criteria as line search conditions, which
greatly accelerate the convergence speed. Moreover, we provide a detailed convergence analysis of the proposed algorithm, showing that
the algorithm converges under both monotone and non-monotone line search criteria. Experiments results on large-scale data sets
demonstrate the fast convergence of the proposed algorithm.

In our future work, we will focus on analyzing the theoretical performance
(e.g., prediction error bound, parameter estimation error bound etc.)
of the solution obtained by the GIST algorithm. In addition, we plan to apply
the proposed algorithm to solve the multi-task feature
learning problem~\cite{gong2012multi,gong2012robust}.

\section*{Acknowledgements}
This work is supported partly by 973 Program (2013CB329503), NSFC (Grant No. 91120301, 61075004, 61021063),
NIH (R01 LM010730) and NSF (IIS-0953662, CCF-1025177, DMS1208952).

\appendix
\normalsize



\section*{Appendix: Solutions to Problem (\ref{eq:minimizesurrogate})}
Observe that $r(\mathbf{w})=\sum_{i=1}^dr_i(w_i)$ and problem (\ref{eq:minimizesurrogate}) can
be equivalently decomposed into $d$ independent univariate optimization problems:
\begin{align}
w^{(k+1)}_i=\argmin_{w_i}h_i(w_i)=\frac{1}{2}\left(w_i-u^{(k)}_i\right)^2+\frac{1}{t^{(k)}}r_i(w_i),\nonumber
\end{align}
where $i=1,\cdots,d$ and $u^{(k)}_i$ is the $i$-th entry of $\mathbf{u}^{(k)} = \mathbf{w}^{(k)} - \nabla l(\mathbf{w}^{(k)})/t^{(k)}$. To
simplify the notations, we unclutter the above equation by removing the subscripts and supscripts as follows:
\begin{align}
w^{(k+1)}=\argmin_{w}h_i(w)=\frac{1}{2}\left(w-u\right)^2+\frac{1}{t}r_i(w).\label{eq:proximalunivariate}
\end{align}
\begin{itemize}
\item $\bm{\ell_1}$-\textbf{norm}: $w^{(k+1)}=\sign(u)\max\left(0,|u|-\lambda/t\right)$.
\item \textbf{LSP}: We can obtain an optimal solution of problem (\ref{eq:proximalunivariate}) via:
$w^{(k+1)}=\sign(u)x$,
where $x$ is an optimal solution of the following problem:
\begin{align}
&x=\argmin_{w}\frac{1}{2}\left(w-|u|\right)^2+\frac{\lambda}{t}\log(1+w/\theta)\nonumber\\
&~~~~~~~~~~s.t.~w\geq 0.\nonumber
\end{align}
Noting that the objective function above is differentiable in the interval $[0,+\infty)$ and
the minimum of the above problem is either a stationary point (the first derivative is zero) or an endpoint of the feasible region, we have
\begin{align}
x=\argmin_{w\in\mathcal{C}}\frac{1}{2}\left(w-|u|\right)^2+\frac{\lambda}{t}\log(1+w/\theta),\nonumber
\end{align}
where $\mathcal{C}$ is a set composed of $3$ elements or $1$ element.
If $t^2(|u|-\theta)^2-4t(\lambda-t|u|\theta)\geq 0$,
\begin{align}
&\mathcal{C}=\left\{0,\right.\nonumber\\
&\left[\frac{t(|u|-\theta)+\sqrt{t^2(|u|-\theta)^2-4t(\lambda-t|u|\theta)}}{2t}\right]_+\nonumber\\
&\left.\left[\frac{t(|u|-\theta)-\sqrt{t^2(|u|-\theta)^2-4t(\lambda-t|u|\theta)}}{2t}\right]_+\right\}.\nonumber
\end{align}
Otherwise, $\mathcal{C}=\left\{0\right\}$.
\item \textbf{SCAD}: We can recast problem (\ref{eq:proximalunivariate}) into the following three problems:
\begin{align}
&x_1=\argmin_{w}\frac{1}{2}\left(w-u\right)^2+\frac{\lambda}{t}|w|\quad s.t.~|w|\leq\lambda,\nonumber\\
&x_2=\argmin_{w}\frac{1}{2}\left(w-u\right)^2\nonumber\\
&+\frac{-w^2+2\theta(\lambda/t)|w|-(\lambda/t)^2}{2(\theta-1)}~s.t.~\lambda\leq |w|\leq\theta\lambda,\nonumber\\
&x_3=\argmin_{w}\frac{1}{2}\left(w-u\right)^2+\frac{(\theta+1)\lambda^2}{2t^2}s.t.|w|\geq\theta\lambda.\nonumber
\end{align}
We can easily obtain that ($x_2$ is obtained using the similar idea as \textbf{LSP} by considering that $\theta>2$):
\begin{align}
&x_1=\sign(u)\min(\lambda,\max(0,|u|-\lambda/t)),\nonumber\\
&x_2=\sign(u)\min(\theta\lambda,\max(\lambda,\frac{t|u|(\theta-1)-\theta\lambda}{t(\theta-2)})),\nonumber\\
&x_3=\sign(u)\max(\theta\lambda,|u|).\nonumber
\end{align}
Thus, we have
\begin{align}
w^{(k+1)}=\argmin_{y}h_i(y) \quad s.t.~y\in\{x_1,x_2,x_3\}.\nonumber
\end{align}
\item \textbf{MCP}: Similar to SCAD, we can recast problem (\ref{eq:proximalunivariate}) into the following two problems:
\begin{align}
&x_1=\argmin_{w}\frac{1}{2}\left(w-u\right)^2+\frac{\lambda}{t}|w|-\frac{w^2}{2\theta}~s.t.~|w|\leq\theta\lambda,\nonumber\\
&x_2=\argmin_{w}\frac{1}{2}\left(w-u\right)^2+\frac{\theta(\lambda/t)^2}{2} ~s.t.~|w|\geq\theta\lambda.\nonumber
\end{align}
We can easily obtain that
\begin{align}
x_1=\sign(u)z,~x_2=\sign(u)\max(\theta\lambda,|u|),\nonumber
\end{align}
where $z=\argmin_{w\in\mathcal{C}}\frac{1}{2}\left(w-|u|\right)^2+\frac{\lambda}{t}w-\frac{w^2}{2\theta}$;
$\mathcal{C}=\left\{0,\theta\lambda,\min\left(\theta\lambda,\max\left(0,\frac{\theta(t|u|-\lambda)}{t(\theta-1)}\right)\right)\right\}$,
if $\theta-1\neq0$, and $\mathcal{C}=\left\{0,\theta\lambda\right\}$ otherwise.
Thus, we have
\begin{align}
w^{(k+1)}=\left\{
                 \begin{array}{ll}
                 x_1, & \mathrm{if}~h_i(x_1)\leq h_i(x_2) \\
                 x_2, & \mathrm{otherwise}.
                 \end{array}
          \right.\nonumber
\end{align}
\item \textbf{Capped} $\bm{\ell_1}$:
We can recast problem (\ref{eq:proximalunivariate}) into the following two problems:
\begin{align}
&x_1=\argmin_{w}\frac{1}{2}\left(w-u\right)^2+\frac{\lambda}{t}\theta \quad s.t.~|w|\geq\theta,\nonumber\\
&x_2=\argmin_{w}\frac{1}{2}\left(w-u\right)^2+\frac{\lambda}{t}|w| \quad s.t.~|w|\leq\theta.\nonumber
\end{align}
We can easily obtain that
\begin{align}
&x_1=\sign(u)\max(\theta,|u|),\nonumber\\
&x_2=\sign(u)\min(\theta,\max(0,|u|-\lambda/t)).\nonumber
\end{align}
Thus, we have
\begin{align}
w^{(k+1)}=\left\{
                 \begin{array}{ll}
                 x_1, & \mathrm{if}~h_i(x_1)\leq h_i(x_2), \\
                 x_2, & \mathrm{otherwise}.
                 \end{array}
          \right.\nonumber
\end{align}

\end{itemize}

\small
\bibliography{icml2013}
\bibliographystyle{icml2013}

\end{document}